\documentclass[letterpaper]{article} 
\usepackage{aaai22} 
\usepackage{times}  
\usepackage{booktabs}
\usepackage{setspace}
\usepackage{helvet}  
\usepackage{amsmath}
\usepackage{subfigure}
\newtheorem{theorem}{Theorem}
\newtheorem{lemma}{Lemma}
\newtheorem{mydef}{Definition}
\newtheorem{proof}{Proof}
\usepackage{multirow}
\usepackage{booktabs}
\usepackage{amssymb}
\usepackage{tablefootnote}
\usepackage{courier}  
\usepackage[hyphens]{url}  
\usepackage{graphicx} 
\usepackage{natbib}  
\usepackage{caption} 
\DeclareCaptionStyle{ruled}{labelfont=normalfont,labelsep=colon,strut=off} 
\usepackage{algorithm}
\usepackage{algorithmic}

\usepackage{newfloat}
\usepackage{listings}
\floatstyle{ruled}
\newfloat{listing}{tb}{lst}{}
\floatname{listing}{Listing}
\title{Multi-Scale Architectures Matter: On the Adversarial Robustness of Flow-based Lossless Compression}
\author{
    Yi-chong Xia, Bin Chen, Yan Feng, Tian-shuo Ge 
}
\affiliations{
    Keywords: Flow models, Lossless Compression,
    Adversarial Attack, Lipschitz Property
}

\begin{document}
\maketitle
\begin{abstract}

Flow-based models are attractive in lossless image compression due to their exact likelihood estimation on entropy coding. In practice, the lossless compression ratio affects the storage cost of data with high fidelity. However, the robustness and efficiency trade-off of flow-based deep lossless compression models has not been fully explored. In this paper, we characterize the trade-off of flow-based deep lossless compression theoretically and empirically. Specifically, we show that flow-based models are susceptible to adversarial examples in terms of dramatic change in compression ratio. In theory, we demonstrate that the fragile robustness of the flow-based model is caused by its intrinsic multi-scale architectures that lack the desirable Lipschitz property. Based on the theoretical insight, we develop a stronger attack, i.e., Auto-Weighted Projected Gradient Descent (AW-PGD), on flow-based lossless compression algorithms, which is more effective and can generate more universal adversarial examples. Furthermore, we propose a novel flow-based lossless compression model, Robust Integer Discrete Flow (R-IDF), which can achieve comparable robustness as adversarial training without loss of compression efficiency. Our experiments show that the PGD algorithm will fall into local extreme values when attacking the compression model. Our attack method can effectively escape the saddle point and attack the compression task more effectively. Moreover, our defense method can greatly improve the invulnerability of the flow-based compression model with almost no sacrificing of clean model performance.

\end{abstract}
\section{Introduction}



As a probabilistic modeling technique, the flow-based model has demonstrated remarkable potential in the field of lossless compression \cite{idf,idf++,lbb,ivpf,iflow},. Compared with other deep generative models (eg. Autoregressive, VAEs) \cite{bitswap,hilloc,pixelcnn++,pixelsnail} that explicitly model the data distribution probabilities, flow-based models perform better due to their excellent probability density estimation and satisfactory inference speed. In flow-based models, multi-scale architecture provides a shortcut from the shallow layer to the output layer, which significantly reduces the computational complexity and avoid performance degradation when adding more layers. This is essential for constructing an advanced flow-based learnable bijective mapping. Furthermore, the lightweight requirement of the model design in practical compression tasks suggests that flows with multi-scale architecture achieve the best trade-off between coding complexity and compression efficiency. 

\begin{figure}[tb]
\begin{center}
\includegraphics[width=0.95\columnwidth]{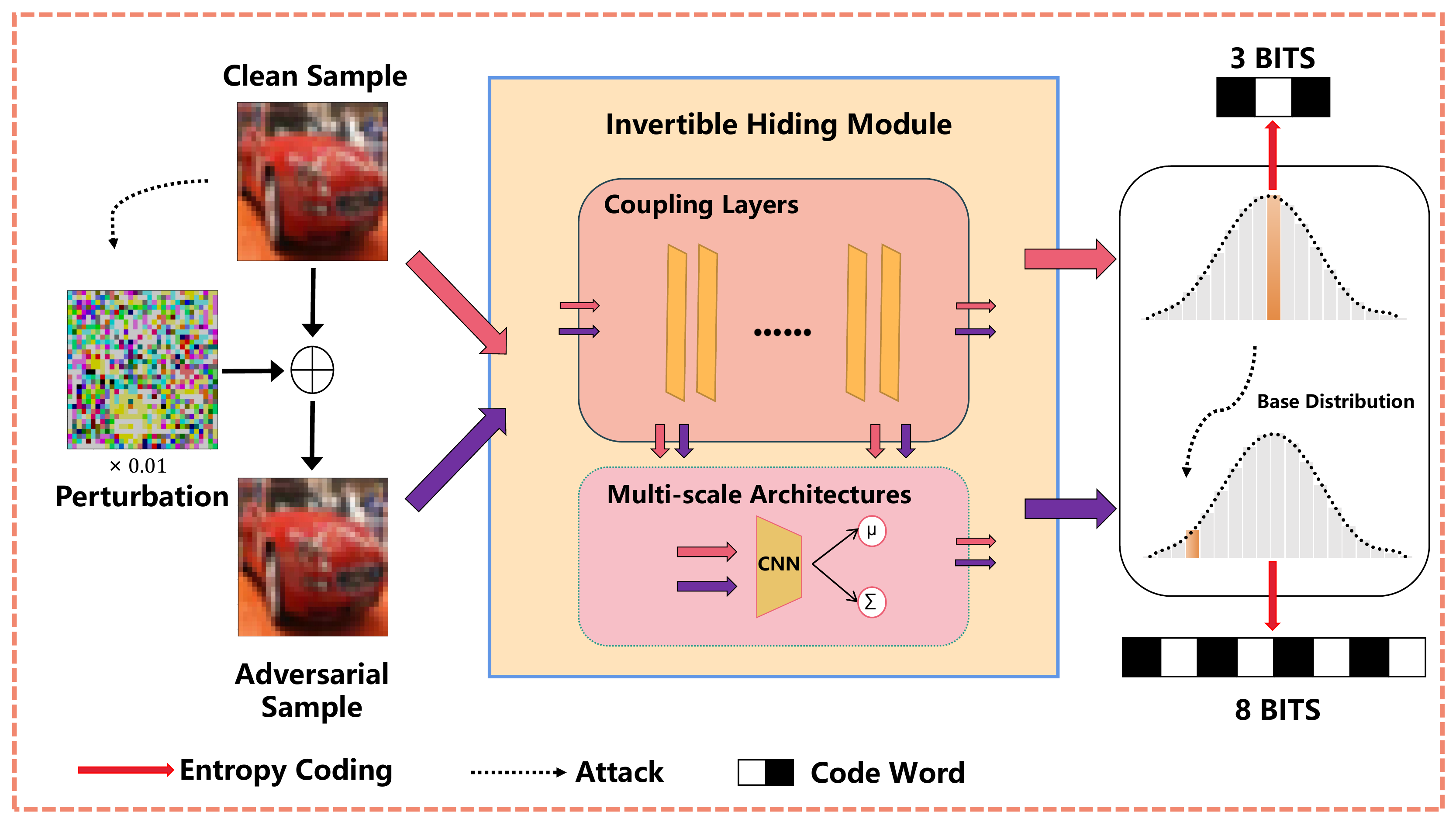}
\end{center}
  \caption{An example of an adversarial attack on a flow-based compression model. Although this adversarial noise is almost invisible, it successfully causes a dramatic shift in the likelihood of the flow model output and reduces the compression ratio from 2.5 to 1 (uncompressed).}
\label{intro}
\end{figure}

Despite recent works \cite{ivpf,idf,idf++} mainly focus on designing more expressive architectures within the flow-based lossless compression models, the robustness of the inner architectures is ignored.
This poses a potential security threat in cloud and streaming platforms with lossless compression codecs consisting of 
unreliable architectures. For example, the cloud and streaming platforms can be vulnerable to flow attack when the volume of transmitted or stored data grows rapidly controlled by the adversary as shown in Fig. 1. The clean sample is assigned with a short code length 3 bits estimated with high probability from a known probability distribution by the flow-based compression model. But the adversary can add a slight perturbation to induce the perturbed image being wrongly estimated with low probability, i.e. with longer code length 5 bits. This makes it problematic to deploy flow-based lossless compression codes in practice. Accordingly, it is desirable to study
the architectural robustness of flow-based lossless compression models and address such security concerns. 
While prior work \cite{adflow} has shown that the flow model might become vulnerable under adversarial attack, there appears to be a restricted approach since their theoretical analysis is only based on the continuous flow models under simplified assumptions. By contrast, a deeper understanding of the architectural vulnerability of flow-based lossless compression models is still missing. Therefore, we raise the following questions:
\emph{Whether or not the inner architecture itself of a flow-based lossless compression model can expose more vulnerable to adversarial attacks? How can we design a robust lossless compression codec from an architectural robustness perspective?}

In this paper, we make the ﬁrst attempt to identify the weakness of the multi-scale architecture adopted by many state-of-the-art flow-based lossless compression models, from an empirical and theoretical perspective. Specifically, We analyze the robustness of the multi-scale architecture via the Lipschitz property, which reveals the reason behind the fragile adversarial robustness. Based on the theoretical insight, we design a more powerful attack, i.e., auto-weighted projected gradient descent (AW-PGD) attack, to expose the more vulnerable endopathic controlled by the multi-Scale architecture. Although the 
multi-scale architecture is more vulnerable to adversarial attacks, we surprisingly find that the robustness of the rest of the model can be enhanced by such weakness. To achieve this we propose a novel lossless compression model named 
robust integer discrete flow (R-IDF). A method for constraining multiscale-architecture.
Experiments show that R-IDF  can significantly enhance the robustness of the model. For instance, the compression ratio is improved from 1.04 to 1.40 under PGD-10 


Our main contributions can be summarized as follows:
\begin{itemize}
\item We present a more reasonable assumption for the flow model. Based on this, we prove an upper bound on the increase of negative log-likelihood under adversarial perturbation. Our theoretical results meets experimental expectations, and expose the architectural vulnerability of flow-based lossless compression models to the adversarial attack.
\item Based on our theoretical analysis on the architectural robustness, we designed AW-PGD, which can effectively help the PGD attack to get out of the saddle point dilemma when generating adversarial samples. Our experiments show that AW-PGD is more effective and universal adversarial examples. 
\item We propose R-IDF, a regularization method for IDF, which achieves both reliable robustness and computational efficiency comparable to adversarial training. Compared with the clean model and the adversarial training model, its robustness is improved by 50\% and 5\%, respectively
\end{itemize}
\begin{figure*}[ht]
\begin{center}
\includegraphics[width=0.95\textwidth]{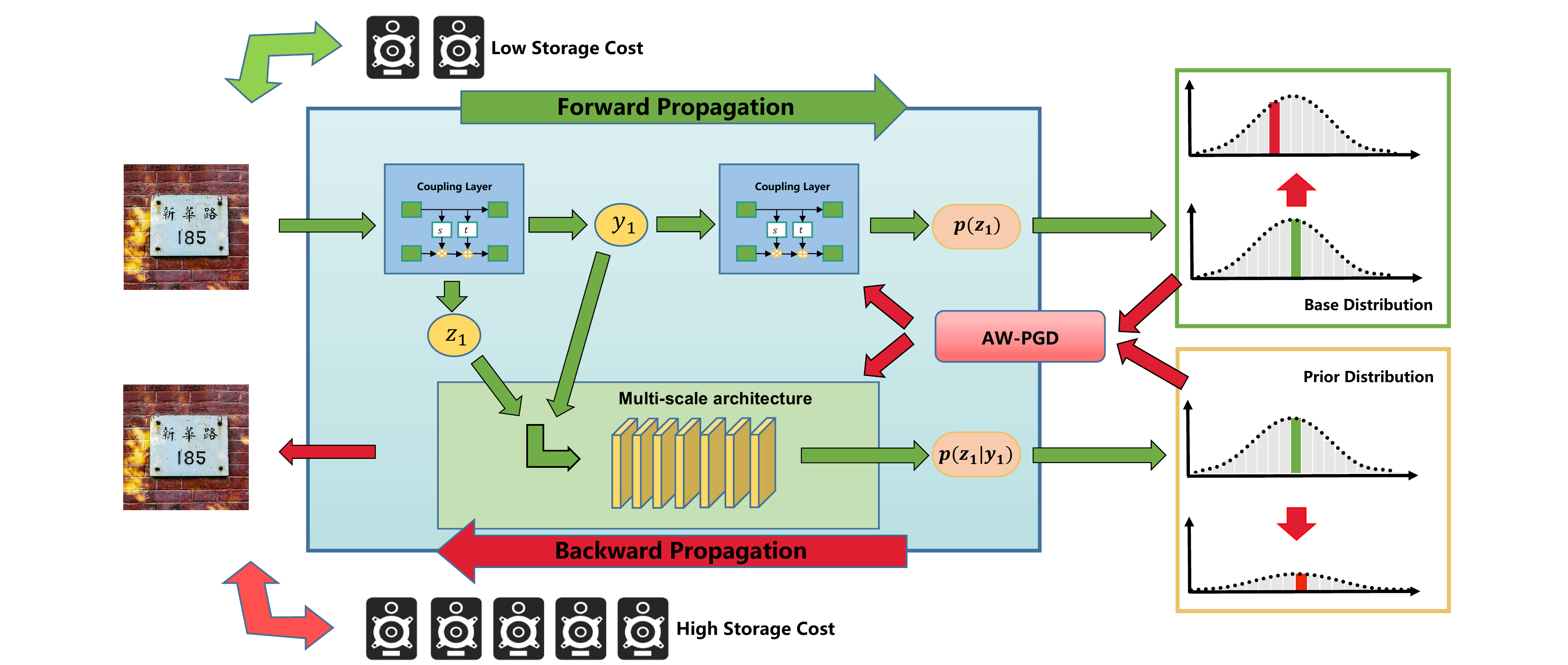}
\end{center}
  \caption{The full-scale pipeline of the proposed AW-PGD method against a flow-based model. Through Multiscale-architecture, the input image outputs multiple prior likelihoods and the likelihood of the base distribution. AW-PGD gives different weights to the gradients of varying likelihood loss to achieve the effect of concentrating on the attack of vulnerable modules. Note that the parameters of the base distribution are fixed, while the prior distribution parameters are predicted by CNN. As such, the adversarial attack works differently when attacking these two distributions.}
\label{pipeline}
\end{figure*}
\section{Related Work}
\subsection{Flow-based lossless compression}
According to Shannon's source coding theory \cite{information}, the limit of lossless compression is the entropy of the data distribution. Many efficient entropy coding algorithms can approach or reach this lower bound \cite{huffman,Arithmetic,duda2013asymmetric}, but only if we can explicitly evaluate the density functions. Deep models that can accomplish it can be divided into three categories. VAEs with Bits-Back technique \cite{hilloc,bitswap} can be used for efficient lossless coding. However, VAEs model ELBO, which leads to a theoretical gap between the actual distribution. This prevents the compression rate from being improved further \cite{ivpf}. Autoregressive models \cite{pixelcnn++,pixelsnail}, such as Pixel-CNN, maintain the best distribution estimation performance. However, the considerable time consumption caused by pixel-by-pixel inference makes the application unattainable. Some recent works implement local Autoregressive, which makes it achievable to accelerate inference for AR \cite{parallel,nelloc}.
Continuous flows    \cite{glow,realnvp,flowpp+,nice,resflow} use a reversible network structure to model the bijection from the image space to the latent space, providing an precise estimation of the distribution and ensuring acceptable speed.

However, continuous flows cannot be directly used in compression tasks because of their numerical loss in quantization \cite{idf,understanding}. Existing solutions for this can be divided into two categories. \cite{lbb,iflow} provides unique coding methods to efficiently encode the numerical loss caused by quantization. These methods can obtain slight quantization loss, but the specific continuous flow model they are based on \cite{flowpp+} has inferior resistance to adversarial noise and is difficult to proceed with adversarial training. Another approach is to restrict the structure of the flow to ensure a strict bijection of the discrete data to the discrete base distribution \cite{idf,idf++,ivpf}. For example, IDF and IDF++ use the rounding operator, while IVPF uses volume-preserving flows. The disadvantage of this approach is that this constraint can confine the expressivity of the transformation.
\subsection{Adversarial attack and defense}
In the context of a classification task, adversarial examples are slight and malevolently designed perturbations that change a model prediction label. Existing adversarial attacks can be roughly separated into two types: white-box attacks \cite{cw,onepixel,fgsm,pgd} and black-box attacks \cite{b1,b2,b3}. In the white-box scene, the parameters of the model can be obtained. While in the black-box setting, the adversary cannot fetch the model's parameters. Here we mainly consider white-box attacks.
Given a labeled data set $\mathcal{D}=\{(\mathbf{x}, y)\}_{i=1}^{N}$, adversarial example is aimed to change the prediction from DNN. (eg.$f(x)=y$, $f(x_{adv})\neq y$), under the premise of guarantees $\left\|\boldsymbol{x}_{a d v}-\boldsymbol{x}\right\|_{l_p} \leq \epsilon$.
The commonly used white-box attack methods include the Fast Gradient Sign Method (FGSM) \cite{fgsm},
$$
\boldsymbol{x}_{a d v}=\boldsymbol{x}+\alpha \cdot \operatorname{sign}\left(\nabla_{\boldsymbol{x}} \ell(f(\boldsymbol{x}), y)\right)
$$
and projected gradient descent (PGD) \cite{pgd},
\begin{equation}
 \!\boldsymbol{x}_{a d v}^{t+1}={\rm Proj}_{\epsilon}\left(\boldsymbol{x}_{a d v}^{t}+\alpha \cdot \operatorname{sign}\left(\nabla_{\boldsymbol{x}} \ell\left(f\left(\boldsymbol{x}_{a d v}^{t}\right), y\right)\right)\right)   
\end{equation}
PGD is equivalent to a multi-iteration version of FGSM and keeps the adversarial samples within the restricted domain utilizing a projection operator. The above attack methods have been verified to have sound effects on multiple data sets and different DNN networks.

There has been extensive research on improving the robustness of DNNs for defending against adversarial examples. One of the defense methods that can effectively withstand adversarial examples is adversarial training. The methodology of adversarial training can be seen as a min-max function optimization process:
$$
\min _{\theta} \mathbb{E}_{(\mathbf{x}, y) \sim \mathcal{D}} \left[ \max _{\boldsymbol{\delta} \in S} \ell(\boldsymbol{\theta}, \boldsymbol{x}+\boldsymbol{\delta}, y)\right]
$$
The inner maximization is equivalent to the generation process of adversarial examples, and outer minimization tries to reduce the loss of the model on adversarial examples.

Nevertheless, little effort has been devoted to the robustness of generative models, especially flow-based models. \cite{doknow} indicates that flow
may map OOD data to subdomains with high probability
densities. \cite{understanding,relaxing} found in the experiment that when the probability distribution of modeling is very complex, the flow will
become pathological; that is, the hidden variable will leak
mass outside of the support \cite{relaxing}. All of these imply the fragile
robustness of flow from different perspectives. \cite{adflow} explores the robustness of the flow-based model for the first time and theoretically points out that both adversarial attacks and adversarial training are practical on flow. It is indicated from experiments that some widely-used flow-based models do not have satisfactory adversarial robustness. However, its theoretical analysis under strong assumptions limits it to finding more effective attack and defense methods.

\section{Preliminaries}
In this section, we begin with a brief introduction to the normalizing flow  and its major components applied in lossless compression.


\subsection{Flow-based
Lossless Compression}
Normalizing flow is generally designed as a learnable bijection $F: \mathcal{X} \rightarrow \mathcal{Z}
$ from an unknown continuous distribution $p(\boldsymbol{x})$ to another tractable prior probability distribution $p(\boldsymbol{z})$, where  $\boldsymbol{x} \in \mathcal{X}$ and $ \boldsymbol{z} \in \mathcal{Z}$. To realize lossless compression with normalizing flows, the input data $ \boldsymbol{x}$ is transformed into the latent variable $\boldsymbol{x}$ that is encoded with entropy coding by a prior probability distribution. Formally, the log likelihood of $p_{\mathcal{X}}(\boldsymbol{x})$ can be calculated by 
$$
\log p_{\mathcal{X}}(\boldsymbol{x})=\log p_{\mathcal{Z}}(\boldsymbol{z}) 
+
\log\left|\operatorname{det}\left(\frac{\partial\boldsymbol{z} }{\partial \boldsymbol{x}^{\top}}\right)\right|
$$


\subsubsection{Coupling layer}
In order to model the bijective function $F(\cdot)$ efficiently, it can be decomposed into a combination of invertible coupling-layers, i.e., $F=f_{L} \circ f_{L-1} \circ \cdots \circ f_{1}$. This requires that each layer in the flow-based model preserves the reversibility, and the determinant of the
Jacobian matrix of each $f_l$, $l=1, 2, \cdots, L$, is easy to calculate. For this purpose, there are two types of known coupling-layers applied to flow-based models:
\begin{itemize}
\item \textbf{Additive coupling layers:}
\begin{equation}
\begin{aligned}
 f(\boldsymbol{x})_{1:s} &=\boldsymbol{x}_{1:s} \\
f(\boldsymbol{x})_{s+1:C} &=\boldsymbol{x}_{s+1:C}+t\left(\boldsymbol{x}_{1:s}\right)
\label{additive}
\end{aligned}
\end{equation}
\item \textbf{Affine Coupling Layers:}
\begin{equation}
\begin{aligned}
f(\boldsymbol{x})_{1:s} &=\boldsymbol{x}_{1:s} \\
f(\boldsymbol{x})_{s+1:C} &=\boldsymbol{x}_{s+1:C} \odot g\left(s\left(\boldsymbol{x}_{1:s}\right)\right)+t\left(\boldsymbol{x}_{1:s}\right)
\label{affine}
\end{aligned}
\end{equation}
\end{itemize}
Where $\boldsymbol{x}=[\boldsymbol{x}_{1:s},\boldsymbol{x}_{s+1:C}]$,  $\odot$ denotes the hadamard product, and $s(\cdot)$ and $t(\cdot)$ are parameterized neural networks. $g$ stands for scaling operation e.g., sigmoid or $\exp(\cdot)$.

Affine coupling layers is commonly considered to achieve better performance \cite{glow} and is adopted by many advanced flow-based models.

\subsubsection{Multi-Scale Architecture}
Note that a bijection demands that  $\boldsymbol{x}$ should have the same dimension as $ \boldsymbol{z}$'s, this suggests that each $f_l$ needs to remain the same. However, stacking more layers 
of $f_l$, $l=1, 2,\cdots, L$, leads to higher computational complexity by the curse of dimensionality. This fact conflicts with the original intention that deeper architecture can better approximate the image probability distribution in lossless compression. Multi-scale architecture provides a solution to this dilemma, which acts as the role of the skip-connection in ResNet. It allows the model to output  $\boldsymbol{z_k}$ with smaller subdimension directly via a factor-out layer after passing through some $f_l$, as shown in the figure.
Specifically, multi-scale structure decomposes the hidden prior distribution $\boldsymbol{z}$ into a product of multiple prior distributions $\{\boldsymbol{z}_i\}_{i=1}^{k}$, i.e. 
\begin{eqnarray}
\log p( \boldsymbol{z})&=&\log p(\boldsymbol{z}_{k+1})
+
\sum_{i=1}^{k}p(\boldsymbol{z}_i|\boldsymbol{z}_{t > i })\\
&=&\log p(\boldsymbol{z}_{k+1})
+
\sum_{i=1}^{k}p(\boldsymbol{z}_i|\boldsymbol{y}_i)
\end{eqnarray}
Where $k$ stands for the number of \emph{factor-out} layers. Many popular flow-based lossless models have been designed with this structure to obtain better convergence and reduced complexity. For simplicity, we call the module that maps $\boldsymbol{x}$ to $\boldsymbol{z}_{k+1}$ as \emph{main-flow} in the sequel.

\subsubsection{Entropy coding}
Given a $n$-dimensional discrete data $\boldsymbol{x}$, its probability density $p(\boldsymbol{x})$ is approximated by some probabilistic modelling technique. Entropy coding $EC(\cdot)$ is defined as an injective map that compress data losslessly to binary strings with minimum code length lower bounded by the entropy of $H(\boldsymbol{x})$  \cite{information}. Furthermore, Kraft and McMillan \cite{codinge1,codinge2}  point out that we can obtain better lossless compression results by more  effective density estimations, including the flow-based lossless compression methods considered in this paper. 

\section{Methodology}

In this section, we provide some theoretical results to support the empirical results on the vulnerability of flow-based models. To understand the architectural robustness of the flow models, we first show that 
the \emph{main-flow} module and \emph{factor-out} module have different robustness by the well-known Lipschitz property. Then, we design a more robust integer discrete flow (R-IDF) based on the global Lipschitz property proved by Berhrmann et al
\cite{understanding}.

\subsection{Robustness on the Multi-Scale Architecture}
Next we introduce the well-known Lipschitz continuity for our subsequent analysis. 

\begin{mydef}{\rm(Lipschitz Continuity)}. \\
A fuction $F(\cdot):\mathbb{R}^{n} \rightarrow \mathbb{R}^{n}$ is called {\rm $L$-Lipschitz continuous} if there exists a constant $L$ such that for any $x_{1}, x_{2} \in \mathbb{R}^{n}$, we have 
\begin{equation}\label{lip}
   \left\|F\left(x_{1}\right)-F\left(x_{2}\right)\right\| \leq L\left\|x_{1}-x_{2}\right\|
\end{equation}
where $L$ is called the Lipschitz constant. 
\end{mydef}

Noticed that \emph{main-flow} and \emph{factor-out} are two different modules of multi-scale architecture, we analyze their theoretical adversarial robustness, separately.

\subsubsection{Robustness on the Factor-out Layer}
We define factor-out layer as: $f_s(\boldsymbol{y},\boldsymbol{z})=\log \left(p_{q_{\phi}}(\boldsymbol{z})\right)$, where
$q_{\phi}$ stands for a prior distribution with parameters determined by $\mu_{\phi}(\boldsymbol{y})$, $\sigma_{\phi}(\boldsymbol{y})$. Here $\mu_{\phi}(\cdot)$ and $\sigma_{\phi}(\cdot)$ are two parameterized neural networks. In practice, $q_{\phi}$ is generally chosen as a  Gaussion distribution or mixed logistic distribution. With the definition of a factor-out layer, we can obtain 
the following lemma, whose proof is provided in the appendix. 

\begin{lemma}
\label{lemma1}
Assume that the non-linear transformations $
\mu_{\phi}(\cdot)$ and $\sigma_{\phi}(\cdot)$ are both $L$-Lipschitz.
$\boldsymbol{y}, \boldsymbol{z}\in\mathcal{X}\subset\mathbb{R}^n$,  $\|\boldsymbol{z}\|_{\infty} $, $\|\mu_{\phi}(\boldsymbol{y})\|_{\infty}$, and $\|\sigma_{\phi}(\boldsymbol{y})^{-1}\|_{\infty}$ are all bound by $b$. Let $\boldsymbol{z}$ and $\boldsymbol{z^{\star}}$ be samples from $q_{\phi}(Z)$ and 
$q^*_{\phi}(Z^*)$, where $q_{\phi}(Z)$, $q^*_{\phi}(Z^*)$ are normal distribution parameterized by $\{\mu_{\phi}(\boldsymbol{y}), \sigma^2_{\phi}(\boldsymbol{y}) \}$ and $\{\mu_{\phi}(\boldsymbol{y^{\star}}), \sigma^2_{\phi}(\boldsymbol{y^{\star}})\}$, respectively. If
$\|\boldsymbol{y}-\boldsymbol{y^{\star}}\|_2\leq \delta$ and $\|\boldsymbol{z}-\boldsymbol{z^{\star}}\|_2\leq \delta$, then 
\begin{equation}
\begin{aligned}
\lvert \log(p_{q_{\phi}}(\boldsymbol{z}))-\log(p_{q_{\phi}^{\star}}(\boldsymbol{z^{\star}}) \rvert \leq L*\delta*\sqrt{n}*O(b^5)
\end{aligned}
\end{equation}
\end{lemma}

Lemma\ref{lemma1} implies that we 
must obey strict requirements to guaratee the robustness of the factor-out layer. For example, we not only require $\mu(\cdot)$, $\phi(\cdot)$ to be $L$-Lipschitz but also control the norm of the outputs. Since the upper bound is
determined by $b^5$, the factor-out layer tends to exhibit insufficient robustness when the size of $b$ is not regulated effectively, which reveals the  fact that even if the robustness of the main-flow can be guaranteed, the factor-out layer is still vulnerable to adversarial attacks as shown in our subsequent experiments. Furthermore, Lemma 1 demonstrates that the lower bound on the robustness also depends on the dimension of $\boldsymbol{z}$. This suggests that the factor-out layer closer to the input, tends to be less robust and more vulnerable to adversarial attack, which is consistent with our experiment results. 


\subsubsection{Robustness on the Main-flow} We define the 
main-flow as $f_m(\cdot): \boldsymbol{x} \mapsto \boldsymbol{z_{k+1}}$. Under the assumption that the latent variable $\boldsymbol{z_{k+1}} \sim \mathcal{N}\left(0,I_n\right)$, we can derive a lower bound on the robustness of the main-flow in the appendix. 
\begin{lemma}
\label{lemma2}
Assume that $f_m(\cdot): \boldsymbol{x} \mapsto \boldsymbol{z}$ is $L$-lipschitz, $\boldsymbol{x}\in \mathcal{X}\subset\mathbb{R}^n$, and $\boldsymbol{z} \sim \mathcal{N}\left(0,I_n\right)$. If $\|\boldsymbol{x}^{\star}-\boldsymbol{x}\|_2\leq \delta $, we have:
 \begin{equation}
\begin{aligned}
\log(p(\boldsymbol{z^{\star}}))-\log(p(\boldsymbol{z}))
\geq
-L\delta \|\boldsymbol{z}\|_{2}-\frac{\delta^2}{2}
\end{aligned}
\end{equation}
where $\boldsymbol{z}=f_m(\boldsymbol{x})$ and  $\boldsymbol{z^{\star}}=f_m(\boldsymbol{x^{\star}})$.
\end{lemma}

Note that when the flow can certify reversibility and
the probability density of the hidden variable can be concentrated around the base distribution, we have $\|f_m(\boldsymbol{x})\|_{2}=O(\|\boldsymbol{x}\|_{2})$. In practice, the norm $\|\boldsymbol{x}\|_{2}$ is generally bound by 1, which provides the adversarial robustness guarantee by the Lipschitz property of $f_m$. Nevertheless, as shown in the following section, if we adopt the affine coupling layers to achieve more acceptable performance, the Lipschitz property of $f_m$ become unattainable. Therefore, we need to consider the efficiency and robustness trade-off to design a 
more reasonable architecture of flow-based model.


\begin{table}[ht]
\centering
\begin{tabular}{ccccc} 
\toprule
     & Affine   & Additive & Multi-scale & BPD   \\ 
\midrule
Glow & $\surd$  & $\times$ & $\surd$     & 3.39  \\
Glow & $\times$ & $\surd$  & $\surd$     & 3.42  \\
Glow & $\times$ & $\surd$  & $\times$    & \textbf{4.38}$\downarrow$  \\
IDF-C & $\times$ & $\surd$  & $\surd$     & 3.3   \\
IDF-C & $\times$ & $\surd$  & $\times$    & \textbf{4.31}$\downarrow$  \\
\toprule
\end{tabular}
\caption{Evaluation for different modules in Glow and IDF-C (the continuous version of IDF)}
\label{multi-scaletest}
\end{table}

Based on Lemma 1 and Lemma 2, we can prove the following theorem on the architectural weakness of the multi-scale architecture:

\begin{theorem}
Given a flow model defined as $F(\cdot)$, assume that $\boldsymbol{x} \in \mathcal{X}\subset\mathbb{R}^n$, $\|\boldsymbol{x}\|_2\leq b \in \mathbb{R}^{+}$. If the main-flow $f_m(\cdot)$ is $L_1$-Lipschitz, factor-out layers are all $L_2$-Lipschitz and $\|\boldsymbol{x}_{pert}-\boldsymbol{x}_{clean}\|_2\leq \delta$ , then we have 
\begin{eqnarray}
&& L(F(\boldsymbol{x}_{pert}))-L(F(\boldsymbol{x}_{clean}))\nonumber\\
&\geq&-\sum_i^kL_2\delta\sqrt{dim_k}*O(b^5)-L_1\delta*O(b)-\frac{\delta^2}{2}
\end{eqnarray}
where k stands for the number of factor-out layers, $dim_k$ is the input's dimension of $k$-th factor-out layer, $L(\cdot)$ denotes the negative log-likelihood.
\end{theorem}
Notice that we are focusing on mappings between discrete distributions, so we are not concerned about the Jacobian determinants' effect. This does not change our conclusions; see the appendix for further discussion.



\subsection{Auto-Weighted PGD Attack (AW-PGD)}

Based on the previous theoretical analysis, we choose Integer Discrete Flow (IDF) as our objective lossless compression model to design a new auto-weighted PGD attack (AW-PGD). More  arguments on the remaining flow-based compression models can be found in the Appendix. Integer Discrete Flow (IDF) is a discrete flow model with multi-scale architecture. It consists of two factor-out modules that output the prior likelihood in stages. Notice that the PGD algorithm can be directly migrated to the flow-based model and obtain a pleasing attack effect. We first reformulate the prior likelihood output by the multiscale architecture as a muti-loss function:
\begin{eqnarray}
LOSS_{IDF}&=&-\sum_{i=1}^{2}\log p(\boldsymbol{z}_i|\boldsymbol{y}_i)-\log p(\boldsymbol{z}_3)\nonumber\\
&=&loss _{fo1}+loss_{fo2}+loss _{mf},
\end{eqnarray}
where $loss _{fo1}, loss_{fo2}$, and $loss _{mf}$ denotes the loss outputted by the $1$-st factor-out layer, $2$-st factor-out layer, and main-flow, respectively. During the training stage, these three losses enjoy the same weight because they all contribute to the distribution estimation of the hidden variables, but the identical weight setting is not reasonable for the design of an adversarial attack. 

In theory, Lemma \ref{lemma2} implies that the adversarial robustness of different modules that output $loss_{fo1}$, $loss_{fo2}$, and $loss_{mf}$ is entirely different, the lower bound of the factor-out Layer’s robustness is influenced by the dimension size of the input variables. This motivates us to assign higher weight to modules with poorer robustness so that the direction of the adversarial attack is more biased towards this module. Therefore, we propose Auto Weight Projected Gradient Descent (AW-PGD), a method for weighting muti-loss in the process of iterative attack. In particular, we calculate the increment of separate loss in the $i$-th iteration of the PGD algorithm and assign particular weights based on this increment, then we obtain the overall loss as follows:
$$LOSS^{i}=w_{\Delta_{fo1}}^iloss^i_{fo1}+w_{\Delta_{fo2}}^iloss^i_{fo2}+w_{\Delta_{mf}}^iloss^i_{mf},$$
where the increment is defined by
$$\Delta_j^i={\rm max}(loss_j^i-loss_j^{i-1},0),$$
and the normalized weights is obtain by the softmax:
$$w_{\Delta_{fo1}}^i,w_{\Delta_{fo2}}^i,w_{\Delta_{mf}}^i = {\rm Softmax}(\Delta_{fo1}^i,\Delta_{fo2}^i,\Delta_{mf}^i).$$

From another perspective, AW-PGD introduce the prior information of module robustness difference to adjust preferences during PGD attack. Note that the IDF ensures a bijective mapping from $\mathbb{Z}^{d}$ to $\mathbb{Z}^{d}$. This denotes that an integer with 8 bits can represent the output component of any intermediate layer from 0 to 255. Therefore, when the AW-PGD algorithm is applied to IDF, we will specify an upper bound of 8 bits per dimension for each loss. Then we 
can set the loss to this upper bound when it exceeds. The full algorithmic process is provided in the \textbf{supplementary material}

\subsection{Defense via Robust Integer Discrete Flow}

To design a more robust integer discrete flow to defend the adversarial attack, we recall the Lipschitz property proved by Behrmann et al [\cite{understanding}, Theorem 2] that 
the affine layers cannot guarantee global Lipschitz property, leading to poorer robustness compared with the additive layers. Therefore, 
we need to employ additive coupling layers to the main-flow structure, which guarantees global Lipschitz property and tighter Lipschitz bounds.
Furthermore, we need to utilize multi-scale architecture to stack as many additive layers as possible to achieve comparable performance. To this end, the principle of our defense is an appropriate adjustment to make the multi-scale architecture robust via a regularization method, termed robust integer discrete flow (R-IDF).

Inspired by Lemma 1 that the local Lipschitz bound of multi-scale architecture is jointly controlled by $\|\mu_{\phi}(\boldsymbol{y})\|_{\infty}$, $\|\sigma_{\phi}(\boldsymbol{y})^{-1}\|_{\infty}$, and the Lipschitz property of the factor-out layer, we use the spectral norm regularization of the convolution kernel in the factor-out layer to ensure its Lipschitz property and make its output norms as small as possible. In practice, spectral norm regularization and weight decay are used to constrain the factor-out layer: 
\begin{equation}
\begin{aligned}
LOSS_{R}=LOSS_{IDF}+\rho_1\|W_{fo}\|_2^2+\rho_2\sigma(W_{fo})^2,
\end{aligned}
\end{equation} 
where $W_{fo}$ represents the parameter matrix of the fact-out layer,  $\sigma(W)$ and $\|W_{fo}\|_2^2$ is the spectral norm, and $\ell_2 $ norm for $W$ with coefficient $(\rho_1,\rho_2) $.

Since two factor-out layers in IDF have different robustness due to input dimension differences, we can obtain substantial robustness gains by applying the above constraint only to the first factor-out layer, and this constraint barely loses accuracy on clean samples. Furthermore, our regularization method can be combined with hybrid adversarial training. As shown in the subsequent experiments, the combination of R-IDF and the hybrid adversarial training can further improve the robustness of the model, making it equal to or even transcend the robustness of adversarial training without sacrificing the accuracy on clean samples. 




\section{Experiment}

\subsection{Experiment setting}
\subsubsection{Datasets and Evaluation Metrics} Following the settings in \cite{doknow,adflow}, we verify the effectiveness of our methods over \textbf{CIFAR-10} \cite{cifar} and \textbf{SVHN} \cite{svhn} datasets. We adopt the standard metric in image explicitly probability modelling tasks and image compression tasks, Bits per dimension (BPD) and compression rate (CR), to evaluate the effectiveness of our method. BPD can be derived from NLL (negative log-likelihood).
When the image has not undergone any compression operation, the BPD is 8. So the compression rate is $8/{\rm BPD}$. (We provide the implementation details in Sec.1 of the \textbf{Supplementary Material}).
\subsection{Attack Evaluation}
In this chapter, we verify the effectiveness of our attack against Integer Discrete Flow (IDF) and some most widely-adopted continuous flow models. Note that we also provide analysis for the transferability of the attack methods in Sec.2 of the \textbf{Supplementary Material}.\\
%
\begin{figure}[ht]
\begin{minipage}[t]{0.5\linewidth}
\centering
\includegraphics[width=1.6in]{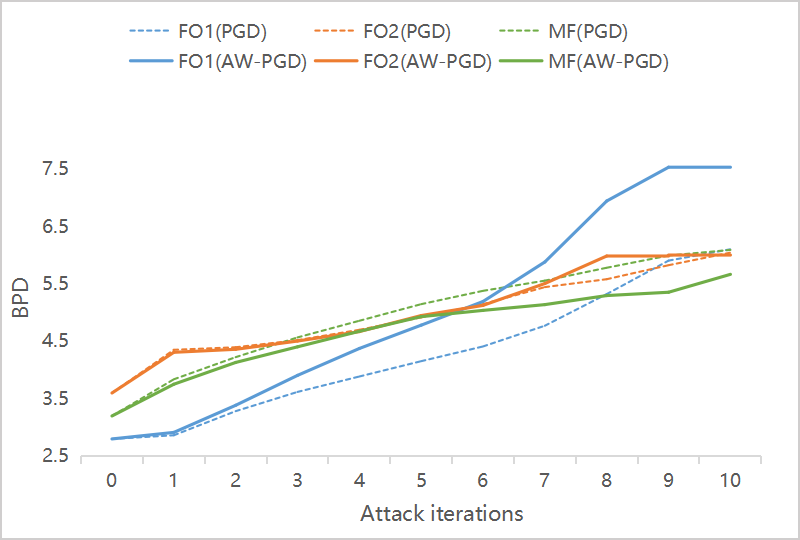}
\caption*{(a) perturbation $\epsilon=1$}
\end{minipage}%
\begin{minipage}[t]{0.5\linewidth}
\centering
\includegraphics[width=1.6in]{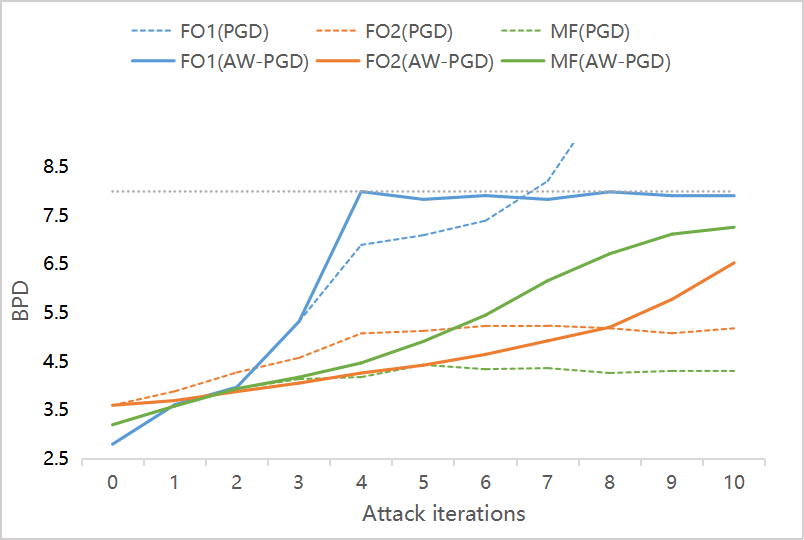}
\caption*{(b) perturbation $\epsilon=2$}
\end{minipage}
\caption{Evaluation results for AW-PGD.} 
\label{fig:IDF attack}
\end{figure}
\vspace{-5mm}
\subsubsection{Attack on IDF}
We first perform our proposed AW-PGD attack against IDF and compare it with the baseline method PGD. Note that due to the compression task, we set the bound in AW-PGD to 8. Fig. \ref{fig:IDF attack} shows that AW-PGD can concentrate the gradient on modules with inferior robustness even with small perturbation (i.e. $\epsilon = 1$). When the perturbation magnitude is higher (i.e. $\epsilon = 2$), the PGD algorithm will obsess over increasing  $loss_{fo1}$ and ignore the gradients of other losses. Nevertheless, as mentioned before, the loss of more than 8-BPD does not result in an effective attack on the compression task. 
The specific attack effect is shown in Tab. \ref{tab:IDF attack}, and we can see that AW-PGD has a significant improvement (38\%) when $\epsilon=1$,  which demonstrates that our attack is extremely effective even over negligible perturbations . When the attack intensity becomes larger (i.e. $\epsilon \geq 2$ ), while the superiority of AW-PGD is partly masked by the fragile robustness of IDF, it still achieves the best results in all the experiments. 

\begin{table}[ht]
\centering
\scalebox{0.86}{
\begin{tabular}{cc|cc} 
\toprule
\multicolumn{2}{c|}{Attack Strength}     & CIFAR10                  & SVHN                      \\ 
\midrule
\multicolumn{2}{c|}{$\epsilon$=0 (Clean sample)}   & 2.41 & 3.79  \\ 
\midrule
\multirow{2}{*}{$\epsilon$=1} & AWPGD             & \textbf{1.24} & \textbf{1.36}  \\
                     & PGD               & 1.39 & 1.42  \\ 
\midrule
\multirow{2}{*}{$\epsilon$=2} & AWPGD             & \textbf{1.04} & \textbf{1.12}  \\
                     & PGD               & 1.05 & 1.17  \\ 
\midrule
\multirow{2}{*}{$\epsilon$=3} & AWPGD~            & \textbf{1.02} & \textbf{1.10}  \\
                     & PGD               & 1.04 & 1.20  \\
\bottomrule
\end{tabular}
}
\caption{Adversarial attack against IDF. \emph{The
evaluation metric is CR, where lower CR
means worse result.}}
\label{tab:IDF attack}
\end{table}

\begin{table*}[tbh]
\centering
\scalebox{0.85}{
\begin{tabular}{ccccccccccccc} 
\toprule
\multicolumn{13}{c}{CIFAR-10}                                                                                         \\ 
\midrule
              & Clean & \multicolumn{3}{c}{Rand-Noise} &  & \multicolumn{3}{c}{PGD} &  & \multicolumn{3}{c}{AW-PGD}  \\ 
\cmidrule{2-5}\cmidrule{7-9}\cmidrule{11-13}
              & $\epsilon$=0   & $\epsilon$=1  & $\epsilon$=2  & $\epsilon$=3              &  & $\epsilon$=1  & $\epsilon$=2   &   $\epsilon$=3   &  & $\epsilon$=1  & $\epsilon$=2   & $\epsilon$=3          \\ 
\midrule
IDF-C         & 3.3   & 4.13 & 4.51 & 5.08             &  & 6.51 & 7.62 &   15.9  &  & \textbf{7.13}  & \textbf{11.32}   & \textbf{18.13}        \\
GLOW          & 3.39  & 4.21 & 4.63 & 5.11             &  & 6.6  & 14.12 &   25.31 &  &\textbf{ 7.62}  & \textbf{17.54}  & \textbf{28.27}        \\
RealNVP       & 3.53  & 4.35     & 4.59     &  4.87           & & 6.82 & 7.47  & 8.14           &  & \textbf{6.91}      & \textbf{7.64}       &  \textbf{8.16 }           \\
Flow++        & 3.21  & 4.06 & 4.58 & 5.02             &  & 6.69 & 9.83  &   10.99 &  & -    & -     & -            \\
Residual flow & 3.12  & 4.03 & 4.42 & 4.99             &  & 5.62 & 7.02  &   8.12  &  & -    & -     & -            \\ 
\midrule
R-IDF         & 3.36  & 4.2      &  4.66   & 4.99                 &  & \underline{4.91} & \underline{5.71}  &   \underline{5.92}  &  & 4.99 & 5.7   & 5.88         \\
\bottomrule
\toprule
\multicolumn{13}{c}{SVHN}                                                                                         \\ 
\midrule
              & Clean & \multicolumn{3}{c}{Rand-Noise} &  & \multicolumn{3}{c}{PGD} &  & \multicolumn{3}{c}{AW-PGD}  \\ 
\cmidrule{2-5}\cmidrule{7-9}\cmidrule{11-13}
              & $\epsilon$=0   & $\epsilon$=1  & $\epsilon$=2  & $\epsilon$=3              &  & $\epsilon$=1  & $\epsilon$=2   &   $\epsilon$=3   &  & $\epsilon$=1  & $\epsilon$=2   & $\epsilon$=3          \\ 
\midrule
IDF-C         & 2.09  & 3.35 & 4.24 & 4.63            &  & 5.05 & 8.20 &  10.43  &  & \textbf{6.18}  & \textbf{9.43}   & \textbf{12.11}       \\
GLOW          & 2.12  & 3.25 & 4.10 & 4.59            &  & 6.23  & 25.46 & 60.71 &  & \textbf{8.91}  & \textbf{27.19}  & \textbf{71.23  }      \\
RealNVP       & 2.30  & 3.15      & 3.85     & 4.38                 &  & 5.12  &  6.82     & 7.89        &  &  \textbf{5.57}     & \textbf{7.26}       & \textbf{8.25 }            \\
\midrule
R-IDF         & 2.17  & 3.3      &  4.01   & 4.57               &  & \underline{4.54} & \underline{5.67}  & \underline{5.79}  &  & 4.54 & 5.65   & 5.80         \\
\bottomrule
\end{tabular}
}
\caption{Evaluate PGD and AW-PGD attacks on continuous flow. The metric used here is BPD, with \emph{higher values representing worse distribution estimation performance} of the model.}
\label{atackc}

\end{table*}


\subsubsection{Attack on continuous flow}
In this section, we will use PGD and AW-PGD for adversarial attacks on the continuous flow model. Note that Flow++ and Residual Flow do not utilise multi-scale architectures, so we did not implement AW-PGD for it.
Firstly, it is shown in Table\ref{atackc} that none of these advanced flow models has acceptable adversarial robustness. Even if the PGD attack method is directly used, the output NLL will approach or even exceed 8 BPD at the attack strength $\epsilon \geq 2$. It indicates that those continuous flow models are challenging to design into robust compression algorithms. In contrast, R-IDF reveals significant robustness advantages under different attack strengths.

Additionally, Tab. \ref{atackc} clearly shows that AW-PGD outperforms PGD on all models with multi-scale architecture. Without specifying the downstream task, the NLL outputted by the continuous flow is not limited as it is for the compressed task. It can be seen that AW-PGD can also show significant advantages when the attack intensity is high . 
Noticed that R-IDF strengthens the robustness of multi-scale architecture through regularization. Therefore, there is no significant difference between AW-PGD and PGD in attacking R-IDF task.
\vspace{-2mm}
\subsection{Defense Evaluation}
We first implemented the two defense methods mentioned in \cite{adflow}, adversarial training and hybrid adversarial training, on IDF. We further evaluate the defense effect of our proposed R-IDF and R-IDF combined with hybrid adversarial training. The ablation study for R-IDF is provided in the Sec.2 of the \textbf{Supplementary Material}.
\begin{table}[ht]
\centering
\scalebox{1}{
\begin{tabular}{cccccc} 
\toprule
\multicolumn{6}{c}{CIFAR-10}                        \\ 
\midrule
$\epsilon$ & CLEAN & Hybrid.~ & Adv & R-IDF & R-IDF$\dag$ \\ 
\midrule
0 & 2.41~ & 2.27~ & 1.85~  & \textbf{2.38}~ & 2.27~          \\
1 & 1.35~ & 1.67~ & 1.70~  & 1.63~ & \textbf{1.78}~          \\
2 & 1.04~ & 1.57~ & 1.60~  & 1.40~ & \textbf{1.63}~          \\
3 & 1.03~ & 1.55~ & 1.57~  & 1.35~ & \textbf{1.57}~          \\
4 & 1.01~ & 1.51~ & 1.51~  & 1.34~ & \textbf{1.52}~          \\
5 & 1.00~ & 1.39~ & 1.44~  & 1.32~ & \textbf{1.44}~          \\ 
\bottomrule
\toprule
\multicolumn{6}{c}{SVHN}                            \\ 
\midrule
$\epsilon$ & CLEAN & Adv.~ & Hybrid & R-IDF & R-IDF$\dag$   \\ 
\hline
0 & 3.79~      & 2.75~      & 3.19~       & \textbf{3.67}~      &  3.18~              \\
1 & 1.42~      & 1.93~      & 1.90~          & 1.76~      &  \textbf{1.94}~              \\
2 & 1.17~      & 1.71~     &  1.68~      & 1.41~     &   \textbf{1.73}~             \\
3 & 1.20~      & \textbf{1.62}~      &  1.54~      & 1.38~      &  1.61~              \\
4 & 1.02~      & 1.53~      &  1.50~      & 1.23~      &  1.53~              \\
5 & 1.00~      & \textbf{1.42}~      &  1.29~      & 1.21~     &   1.38~   \\
\bottomrule
\end{tabular}}

\caption{Evaluation results for defense.  The evaluation metric is CR, \emph{where lower CR means  worse result.} $\dag$ stands for combination with hybrid adversarial training.}
\label{defense}
\end{table}

From Tab. \ref{defense}, we can see that R-IDF has the highest compression ratio among all defense methods under clean samples and only decreases by about 1\% compared to the clean model. It is worth noting that the training time of R-IDF (6.5min/epoch) is only 10 \% compared with the adversarial training method (69min/epoch).  However, its robustness is not significantly decreased compared with the adversarial training method. The R-IDF model combined with hybrid adversarial training can further improve the robustness, reaching or even exceeding the robustness of adversarial training. The compression rate on clean samples is improved by about 22\% compared with adversarial training. Further, the adversarial training for Flow++ is 1100min/epoch, which is 170x and 16x that of R-IDF and Hybrid R-IDF, respectively. This shows that the effectiveness of adversarial training arises from unbearable training time cost.
\section{Conclusion}
We consider the architectural robustness of flow-based lossless compression models, focusing on the effect of the inner structures under adversarial attacks. Then we provide the theoretical analysis of the fragile robustness via the Lipschitz property. Inspired by the theoretical insights, we provide a new attack, AW-PGD, which often outperforms PGD. We also propose a new defense, R-IDF, a regularization method for multi-scale architecture, which is significantly more robust than other flow models while maintaining comparable model performance. On the experimental side, we verify that our proposed attack and defense methods are superior to the previous work. To the best of our knowledge, this is the ﬁrst lossless compression method that takes the architectural robustness into account.

\bibliography{aaai22}

\title{Supplementary Material:
Multi-Scale Architectures Matter:
On the Adversarial Robustness of Flow-based}
\onecolumn

\section{Supplementary Material: \\
Multi-Scale Architectures Matter: \\
On the Adversarial Robustness of Flow-based Lossless Compression}

\subsection{Contente Outline}
This is the supplementary of "Multi-Scale Architectures Matter: On the Adversarial Robustness of Flow-based Lossless Compression". The content of this manuscript will be organized as following:
\begin{itemize}
\item \textbf{In Section 1} we provide details for experiment setting.
\item  \textbf{In Section 2} we provide the experimental evaluation of attack universality and ablation experiments of R-IDF.
\item \textbf{In Section 3} we provide the full pseudocode of the AW-PGD algorithm.
\item \textbf{In Section 4} we provide the details of the proofs of our main theoretical results on the robustness of flow models.
\item \textbf{In Section 5}
we provide more visualizations for the experiments in the main text.
\end{itemize}

\subsection{Section 1: Experiment Detail}
\textbf{IDF, RealNVP} were
trained with default values given in their respective implementations. For \textbf{Glow} model evaluated on CIFAR10, we used the pretrained model. For \textbf{Glow} model evaluated on SVHN, we trained with default values given in the \cite{glow}. \textbf{Residual-flow} used pre-train model provided by \cite{resflow}. For \textbf{Flow++}, due to the limitation of GPU, we only use 4 batch-size while original setting is 64. Further, we implemented flow++ with only 400 training epochs. Although the model does not converge, according to \cite{flowpp+}, this already reflects the nature of the model.

We use $\ell_{\infty}$ as the bound in all attack and defense evaluations. For attack evaluation,we both use $iter=10$ for PGD and AW-PGD. For defense evaluation, we select $(\rho_1,\rho_2)=(2, 0.5)$ as the regularization coefficient of R-IDF. Further,  Adversarial training, Hybrid training, R-IDF(Hybrid), were all trained with $iter=10$ and $\epsilon=5$ attack iterations. According to \cite{adflow}, the mixing rate in Hybrid and R-IDF(Hybrid) are set to 0.5. 

All the above experiments were training on two 2080Ti and all models are implemented in the PyTorch framework.
\vspace{-2mm}
\subsection{Section 2: More Experimental Results}
\subsubsection{Universality of adversarial samples}
Fig.\ref{fig:universal} shows the universality of adversarial samples generated by AW-PGD. 
We are doing the following experiments: 1) randomly pick one of the samples $\boldsymbol{x}$ from all test sets. 2) Generate adversarial noise $\boldsymbol{x_{\Delta}}$ using that sample $\boldsymbol{x}$. 3) Transfer the adversarial sample to other samples. 4) Calculate the perturbation of the image except $\boldsymbol{x}$ after being disturbed by $\boldsymbol{x_{\Delta}}$. 5) Repeat the above operation twenty times to calculate the average perturbation.
\begin{figure}[ht]
\begin{center}
\includegraphics[width=0.9\columnwidth]{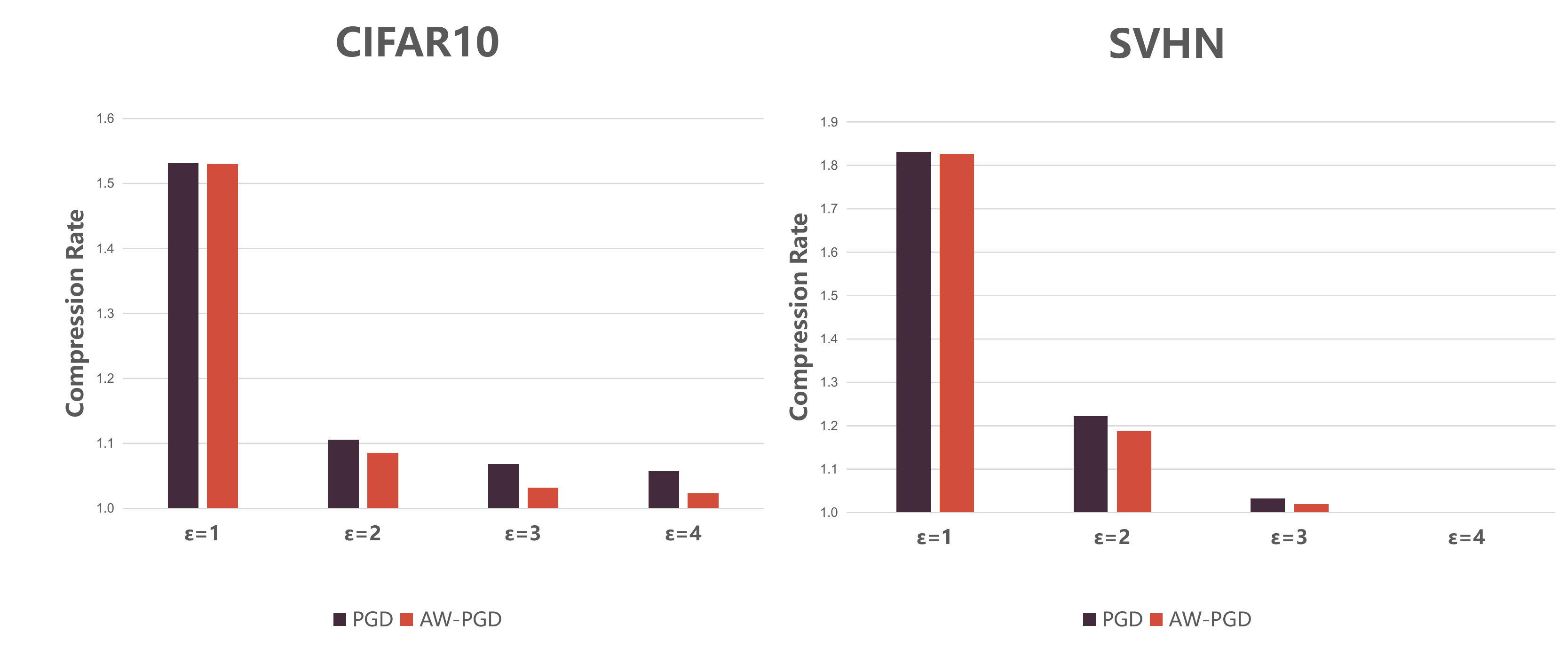}
\end{center}
  \caption{Visualization results for universality}
\label{fig:universal}
\end{figure}

As shown in the Table\ref{tabel:universal}, it is evident that the adversarial noise generated by AW-PGD lead to better universality. Although this phenomenon becomes less significant as the attack strength are too low,  it still indicates that we can significantly reduce the compression effect of the model by transferring the weak adversarial noise generated on only part of the data to the rest of the data.
\begin{table}[ht]
\centering
\setlength{\tabcolsep}{5mm}{\scalebox{0.8}{
\begin{tabular}{ccc|cc} 
\toprule
    & \multicolumn{2}{c|}{CIFAR10} & \multicolumn{2}{c}{SVHN}  \\ 
\midrule
    & PGD     & AW-PGD             & PGD     & AW-PGD          \\ 
\midrule
$\epsilon$=1 & 1.5316~ & \textbf{1.5296}~            & 1.8307~ & \textbf{1.8265}~         \\
$\epsilon$=2 & 1.1058~ & \textbf{1.0852}~            & 1.2227~ & \textbf{1.1871}~         \\
$\epsilon$=3 & 1.0681~ & \textbf{1.0322}~            & 1.0325~ & \textbf{1.0192}~         \\
$\epsilon$=4 & 1.0573~ & \textbf{1.0234}~            & 1.0013~ & \textbf{1.0005}~         \\
$\epsilon$=5 & 1.0000~       & 1.0000~                  & 1.0000~        & 1.0000~               \\
\bottomrule
\end{tabular}}}
\caption{Evaluation results for universality. The evaluation metric is CR, \emph{where lower CR means  worse result.}}
\label{tabel:universal}
\end{table}

\subsubsection{Ablation Study}
We ran the following ablations of our model: IDF vs R-IDF(weight decay only) vs R-IDF (spectral regularization only) vs R-IDF. The training parameters of all models are the same as above. As we can see in Figure\ref{Ablation}, the most significant effect is weight decay. Furthermore, the combination further improves robustness, indicating that the regularized constraints used in R-IDF are practical. It is also compatible with our theoretical analysis.
\begin{figure}[ht]
\begin{center}
\includegraphics[width=0.75\columnwidth]{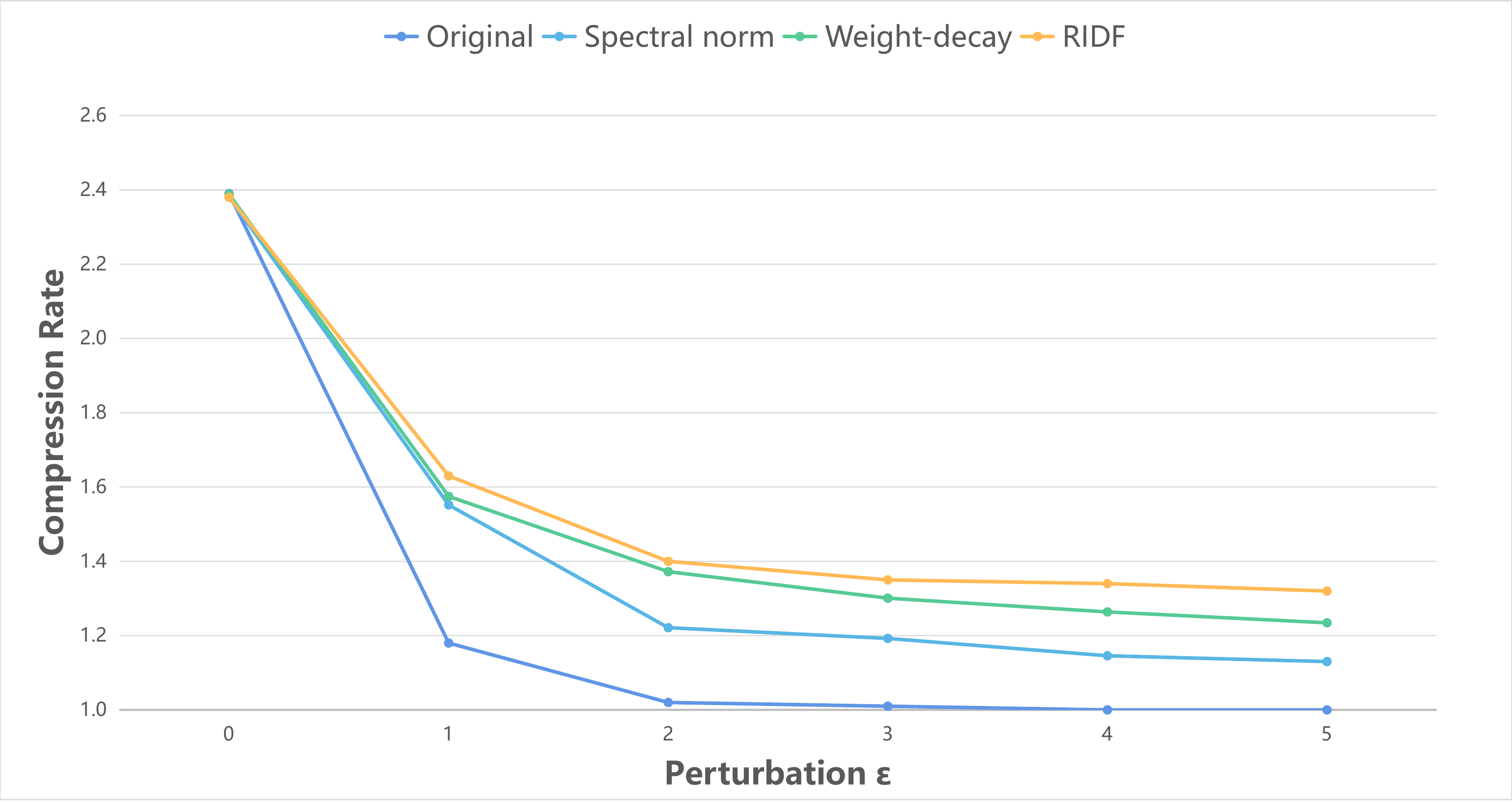}
\end{center}
  \caption{Ablation for R-IDF}
\label{Ablation}
\end{figure}
\vspace{-5mm}
\subsection{Section 3: Algorithm Outline}

We present the idea of the AW-PGD algorithm in the main text, and in this chapter, we will detail the workflow of the AW-PGD. AW-PGD introduce the prior information of module robustness difference by weighting muti-loss in the process of iterative attack. 
\begin{algorithm}[H]
    \caption{Auto-Weighted Projected Gradient Descent}
    {\bf Input:} 
The flow-base model $F(\cdot)$, the clean sample $\boldsymbol{x}$, attack strenth $\epsilon$, the number of iterations $m$, the steps size $\alpha$, restriction of image compression $bound$

\textbf{Output:}
adversarial sampe $\boldsymbol{x}^{\star}=\boldsymbol{x}+\delta_{\boldsymbol{x}}$

\begin{algorithmic}
  \STATE $loss_{1}^0,loss_{2}^0,loss_{3}^0= F(x)$
  \WHILE{$i=1,\dots ,m$}
    \STATE $loss_{1}^i,loss_{2}^i,loss_{3}^i= F(\boldsymbol{x}+ \delta_{\boldsymbol{x}}^i)$
    \IF{$loss_j>bound$}
        \STATE$loss_j=bound$
    \ENDIF
    \STATE$\Delta_j^i={\rm max}(loss_j^i-loss_j^{i-1},0)$
    \STATE$w_{\Delta_{1}}^i,w_{\Delta_{2}}^i,w_{\Delta_{3}}^i = {\rm Softmax}(\Delta_{1}^i,\Delta_{2}^i,\Delta_{3}^i)$
    \STATE $LOSS^i=\sum _{j=1}^3 w_{\Delta_{j}}^iloss_j^i$
    \STATE $\delta^i_{\boldsymbol{x}}=
    {\rm Proj}\left(\delta_{\boldsymbol{x}}^{i-1}+\alpha\cdot sign\left(
    \nabla_{\theta} LOSS^i\right)\right)$
  \ENDWHILE
\end{algorithmic}
\label{alg}
\end{algorithm}

\begin{spacing}{1.5}
\subsection{Section 4: Proofs of Main Theoretical Results on the Robustness of  Flow-based models}

\begin{lemma}:
Assume the non-linear transformations $
\mu_{\phi}(\cdot)$ and $\sigma_{\phi}(\cdot)$are both L-Lipschitz.
$\boldsymbol{x}, \boldsymbol{z}\in \mathcal{X},\mathcal{X}\subset\mathbb{R}^n$. $\|\mathbf{z}\|_{\infty}\leq b_1$, $\|\mu_{\phi}(\boldsymbol{x})\|_{\infty}\leq b_2$, $\|\sigma_{\phi}(\boldsymbol{x})^{-1}\|_{\infty}\leq b_3$,
$b_1,b_2,b_3 \in \mathbb{R}^{+}$
, let $b={\rm max}(b_1,b_2,b_3)$. Finally, let $\boldsymbol{z}$ and $\boldsymbol{z^{\star}}$ be sampled from $q_{\phi}(Z)$ and
$q^*_{\phi}(Z^*)$, where $q_{\phi}(Z)$, $q^*_{\phi}(Z^*)$ are normal distribution parameterized by $\{\mu_{\phi}(\boldsymbol{y}), \sigma^2_{\phi}(\boldsymbol{y}) \}$ and $\{\mu_{\phi}(\boldsymbol{y^{\star}}), \sigma^2_{\phi}(\boldsymbol{y^{\star}})\}$, respectively. If
$\|\boldsymbol{y}-\boldsymbol{y^{\star}}\|_2\leq \delta$ and $\|\boldsymbol{z}-\boldsymbol{z^{\star}}\|_2\leq \delta$, $\delta \in \mathbb{R}^{+}$, then we have 
$$\lvert \log(p_{q_{\phi}}(\boldsymbol{z}))-\log(p_{q_{\phi}^{\star}}(\boldsymbol{z^{\star}}) \rvert \leq L*\delta*\sqrt{n}*C(B)=L*\delta*\sqrt{n}*O(b^5)$$
where $C(B)=C_1(B)+C_2(B)$,
$C_1(B)=(b_1+b_2)^2b_3^3+(b_1+b_2)b_3^2+b_3$, and 
$C_2(B)=(b_1+b_2)b_3$.
\label{lemma1}
\end{lemma}
\begin{proof}
The original can be split into two parts by the absolute value inequality
$$\lvert \log(p_{q_{\phi}}(\boldsymbol{z}))-\log(p_{q_{\phi}^{\star}}(\boldsymbol{z^{\star}}) \rvert \leq \lvert \log(p_{q_{\phi}}(\boldsymbol{z}))- \log(p_{q_{\phi}}(\boldsymbol{z^{\star}}))\rvert +
\lvert \log(p_{q_{\phi}}(\boldsymbol{z^{\star}}))-\log(p_{q_{\phi}^{\star}}(\boldsymbol{z^{\star}}) \rvert$$
\textbf{Firstly}, we need to prove:
$$\lvert \log(p_{q_{\phi}}(\boldsymbol{z^{\star}}))-\log(p_{q_{\phi}^{\star}}(\boldsymbol{z^{\star}}) \rvert \leq l*\delta*\sqrt{n}*C_1(B) $$
Since $q_{\phi}$ and $q_{\phi}^{\star}$ are both the normal distribution of the diagonal matrix with the covariance matrix,we have:\\
\begin{equation*}
\begin{aligned}
\left| \log(p_{q_{\phi}}(\boldsymbol{z^{\star}}))-\log(p_{q_{\phi}^{\star}}(\boldsymbol{z^{\star}}) \right|
&\leq
\left| \sum_{i=1}^{n}\log\left(\frac{\sigma_{\phi}(\boldsymbol{x})_i}{\sigma_{\phi}(\boldsymbol{x^{\star}})_i}\right)
\right|
+\left|
\sum_{i=1}^{n} \frac{(z_i^{\star}-\mu_{\phi}(\boldsymbol{x})_i)^2}{2\sigma_{\phi}(\boldsymbol{x})_i^2}- \frac{(z_i^{\star}-\mu_{\phi}(\boldsymbol{x^{\star}})_i)^2}{2\sigma_{\phi}(\boldsymbol{x^{\star}})_i^2}\right|
\\
 &\leq
 \left| \sum_{i=1}^{n}\log\left(\frac{\sigma_{\phi}(\boldsymbol{x})_i}{\sigma_{\phi}(\boldsymbol{x^{\star}})_i}\right)
\right|
 +\left|
 \sum_{i=1}^{n} (z_i^{\star}-\mu_{\phi}(\boldsymbol{x})_i)^2
 \left( \frac{1}{2\sigma_{\phi}(\boldsymbol{x})_i^2}-\frac{1}{2\sigma_{\phi}(\boldsymbol{x^{\star}})_i^2}
 \right)\right|\\
 &+\left| \sum_{i=1}^{n}
 \frac{1}{2\sigma_{\phi}(\boldsymbol{x^{\star}})_i^2}
 \left((z_i^{\star}-\mu_{\phi}(\boldsymbol{x})_i)^2-(z_i^{\star}-\mu_{\phi}(\boldsymbol{x^{\star}})_i)^2 \right)\right|
\end{aligned}
\end{equation*}
For the first part,we have:

\begin{equation*}
\begin{aligned}
\left| \sum_{i=1}^{n}\log\left(\frac{\sigma_{\phi}(\boldsymbol{x})_i}{\sigma_{\phi}(\boldsymbol{x^{\star}})_i}\right)
\right|
&\leq
 \left|
 \sum_{i=1}^{n}(\sigma_{\phi}(\boldsymbol{x})_i-\sigma_{\phi}(\boldsymbol{x^{\star}})_i)\cdot
 (\frac{1}{\sigma_{\phi}(\boldsymbol{x})_i})_{max}
 \right|
 \\
&\leq
\sqrt{nb_3^2
\sum_{i=1}^{n}(\sigma_{\phi}(\boldsymbol{x})_i-\sigma_{\phi}(\boldsymbol{x^{\star}})_i)^2
}\\
&\leq
\sqrt{n}b_3\|\sigma_{\phi}(\boldsymbol{x})-\sigma_{\phi}(\boldsymbol{x^{\star}})\|_2=\sqrt{n}l\delta b_3
\end{aligned}
\end{equation*}
where the second inequality follows from the Cauchy inequality and Lagrange mean value theorem.

For the second part,we have:

\begin{equation*}
\begin{aligned}
\left|
\sum_{i=1}^{n} (z_i^{\star}-\mu_{\phi}(\boldsymbol{x})_i)^2
 \left( \frac{1}{2\sigma_{\phi}(\boldsymbol{x})_i^2}-\frac{1}{2\sigma_{\phi}(\boldsymbol{x^{\star}})_i^2}
 \right)
 \right|
 &\leq
\sum_{i=1}^{n}\left|
 (z_i^{\star}-\mu_{\phi}(\boldsymbol{x})_i)^2
 \right|
 \left| \frac{1}{2\sigma_{\phi}(\boldsymbol{x})_i^2}-\frac{1}{2\sigma_{\phi}(\boldsymbol{x^{\star}})_i^2}
 \right|
 \\
 &\leq \sqrt{n}b_3^3(b_1+b_2)^2\sqrt{\sum_{i=1}^{n}(\sigma_{\phi}(\boldsymbol{x})_i-\sigma_{\phi}(\boldsymbol{x^{\star}})_i)^2}
 \\
 &\leq
   \sqrt{n}b_3^3(b_1+b_2)^2 \|\sigma_{\phi}(\boldsymbol{x})-\sigma_{\phi}(\boldsymbol{x^{\star}})\|_2
 \\
 &\leq
   \sqrt{n}l\delta b_3^3(b_1+b_2)^2
\end{aligned}
\end{equation*}
where the second inequality follows from the Cauchy inequality and Lagrange mean value theorem.

For the third part,we have:

\begin{equation*}
\begin{aligned}
\left|
\sum_{i=1}^{n}
 \frac{1}{2\sigma_{\phi}(\boldsymbol{x^{\star}})_i^2}
 \left((z_i^{\star}-\mu_{\phi}(\boldsymbol{x})_i)^2-(z_i^{\star}-\mu_{\phi}(\boldsymbol{x^{\star}})_i)^2 \right)
 \right|
 &\leq
 \left|
 \frac{b_3^2}{2}\sum_{i=1}^{n}
 2(\mu_{\phi}(\boldsymbol{x})_i-\mu_{\phi}(\boldsymbol{x^{\star}})_i)\cdot
(z_i^{\star}-\mu_{\phi}(\boldsymbol{x})_i)_{max}
 \right|
 \\
 &\leq
 b_3^2\sqrt{n(b_1+b_2)^2
 \sum_{i=1}^{n}(\mu_{\phi}(\boldsymbol{x})_i-\mu_{\phi}(\boldsymbol{x^{\star}})_i)^2
 }
 \\
 &\leq
 \sqrt{n}(b_1+b_2)b_3^2\|\mu_{\phi}(\boldsymbol{x})-\mu_{\phi}(\boldsymbol{x^{\star}})\|_2
 \\
 &\leq
 \sqrt{n}l\delta (b_1+b_2)b_3^2
\end{aligned}
\end{equation*}
where the second inequality follows from the Cauchy inequality.

Hence,we obtain:
$$\lvert \log(p_{q_{\phi}}(\boldsymbol{z^{\star}}))-\log(p_{q_{\phi}^{\star}}(\boldsymbol{z^{\star}}) \rvert
\leq
\sqrt{n}l\delta b_3+
\sqrt{n}l\delta b_3^3(b_1+b_2)^2+
 \sqrt{n}l\delta (b_1+b_2)b_3^2=
l*\delta*\sqrt{n}*C_1(B) $$
\textbf{Secondly},we need to prove:
$$
\lvert \log(p_{q_{\phi}}(\boldsymbol{z}))- \log(p_{q_{\phi}}(\boldsymbol{z^{\star}}))\rvert
\leq C_2(B)
$$
According to the definition of Lipshitz continuous,
$$
\|f\left(\boldsymbol{x}_{1}\right)-f\left(\boldsymbol{x}_{2}\right)\|_2
\leq K\left\|\boldsymbol{x}_{1}-\boldsymbol{x}_{2}\right\|_{2}
$$
is equivalent to
$$
(\forall \boldsymbol{x} \in \mathcal{X})\|\partial f / \partial \boldsymbol{x}\|_{2} \leq K
$$
Since
\begin{equation*}
\begin{aligned}
(\forall \boldsymbol{x} \in \mathcal{X})
\|\partial \log(p_{q_{\phi}}(\boldsymbol{z})) / \partial \boldsymbol{z}\|_{2}
&=
(\forall \boldsymbol{x} \in \mathcal{X})
\| \operatorname{diag}\left(\sigma_{\phi}^{2}(\boldsymbol{x})\right)^{-1}(\boldsymbol{z}-\mu_{\phi}(\boldsymbol{x}))\|_2
\\
&\leq
\sqrt{\sum_{i=1}^{n}b_3^2(b_1+b_2)^2}=\sqrt{n}(b_1+b_2)b_3
\end{aligned}
\end{equation*}
As such
$$
\lvert \log(p_{q_{\phi}}(\boldsymbol{z}))- \log(p_{q_{\phi}}(\boldsymbol{z^{\star}}))\rvert
\leq
\sqrt{n}(b_1+b_2)b_3\|\boldsymbol{z}-\boldsymbol{z^{\star}}\|_2=l\delta\sqrt{n}(b_1+b_2)b_3
$$
\textbf{Finally},we obtain:
\begin{equation*}
\begin{aligned}
\lvert
\log(p_{q_{\phi}}(\boldsymbol{z}))-\log(p_{q_{\phi}^{\star}}(\boldsymbol{z^{\star}}) \rvert
&\leq
\lvert \log(p_{q_{\phi}}(\boldsymbol{z}))- \log(p_{q_{\phi}}(\boldsymbol{z^{\star}}))\rvert +
\lvert \log(p_{q_{\phi}}(\boldsymbol{z^{\star}}))-\log(p_{q_{\phi}^{\star}}(\boldsymbol{z^{\star}}) \rvert
\\
&=
l\delta\sqrt{n}(C_1(B)+C_2(B))=l\delta\sqrt{n}C(B)
\end{aligned}
\end{equation*}
\end{proof}
\begin{lemma}
Assume that $f_m(\cdot): \boldsymbol{x} \mapsto \boldsymbol{z}$ is $L$-lipschitz, $\boldsymbol{x}\in \mathcal{X}\subset\mathbb{R}^n$, and $\boldsymbol{z} \sim \mathcal{N}\left(0,I_n\right)$. If $\|\boldsymbol{x}^{\star}-\boldsymbol{x}\|_2\leq \delta $, we have:
 \begin{equation*}
\begin{aligned}
\log(p(\boldsymbol{z^{\star}}))-\log(p(\boldsymbol{z}))
\geq
-L\delta \|\boldsymbol{z}\|_{2}-\frac{\delta^2}{2}
\end{aligned}
\end{equation*}
where $\boldsymbol{z}=f_m(\boldsymbol{x})$ and  $\boldsymbol{z^{\star}}=f_m(\boldsymbol{x^{\star}})$.
\label{lemma2}
\end{lemma}
\begin{proof}
Using Theorem 3.1 in \cite{adflow}, we obtain:
\begin{equation*}
\begin{aligned}
\log(p(\boldsymbol{z^{\star}}))-\log(p(\boldsymbol{z}))&=
\log(p_{\mathcal{N}\left(0,I_n\right)}(f_m(\boldsymbol{x^{\star}})))-\log(p_{\mathcal{N}\left(0,I_n\right)}(f_m(\boldsymbol{x})))
\\&\geq
\frac{\|f_m(\boldsymbol(x))\|_2^2}{2}
-\frac{(\|f_m(\boldsymbol(x))\|_2^2)(1+\frac{l\delta}{\|f_m(\boldsymbol(x))\|_2})^2} {2}
\\
&=-(l\delta\|f_m(\boldsymbol{x^{\star}})\|_2+\frac{\delta^2}{2})=-L\delta \|\boldsymbol{z}\|_{2}-\frac{\delta^2}{2}
\end{aligned}
\end{equation*}
Since transformation $f_m$ is $L$- Lipschitz,and $\|x\|_2$ is bounded, so $f_m(x)$ is always bounded. Without loss of generality, $\|f_m(\boldsymbol{x})\|_2 \leq L*\|\boldsymbol{x}-\boldsymbol{0}\|_2 + \|f_m(\boldsymbol{0})\|_2=O(b)$
\end{proof}
\begin{theorem}
Given a flow model defined as $F(\cdot)$, assume that $\boldsymbol{x} \in \mathcal{X}\subset\mathbb{R}^n$, $\|\boldsymbol{x}\|_2\leq b \in \mathbb{R}^{+}$. If the main-flow $f_m(\cdot)$ is $L_1$-Lipschitz, factor-out layers are all $L_2$-Lipschitz and $\|\boldsymbol{x}_{pert}-\boldsymbol{x}_{clean}\|_2\leq \delta$ , then we have
\begin{eqnarray}
 L(F(\boldsymbol{x}_{pert}))-L(F(\boldsymbol{x}_{clean}))\nonumber\geq-\sum_i^kL_2\delta\sqrt{dim_k}*O(b^5)-L_1\delta*O(b)-\frac{\delta^2}{2}
\end{eqnarray}
where k stands for the number of factor-out layers, $dim_k$ is the input's dimension of $k$-th factor-out layer, $L(\cdot)$ denotes the log-likelihood.
\end{theorem}
\begin{proof}
Since:
\begin{eqnarray*}
\begin{aligned}
L(F(\boldsymbol{x}_{pert}))-L(F(\boldsymbol{x}_{clean}))=&(\sum_{i=1}^{k}(\log p(\boldsymbol{z}_i^{pert}|\boldsymbol{y}^{pert}_i)-\log p(\boldsymbol{z}_i^{clean}|\boldsymbol{y}_i^{clean})))\\
&+(\log p(\boldsymbol{z}_{k+1}^{pert})-\log p(\boldsymbol{z}_{k+1}^{pert}))\\
&=\sum_{i=1}^k (Loss_{f_{foi}}(\boldsymbol{x}_{pert})- Loss_{f_{foi}}(\boldsymbol{x}_{clean}))\\
&+ Loss_{mf}(\boldsymbol{x}_{pert})-Loss_{mf}(\boldsymbol{x}_{clean})
\end{aligned}
\end{eqnarray*}
Assume $\boldsymbol{y_t}$ and $\boldsymbol{z_t}$ is the input of t-th factor-out layer$f_{fo_t}$. Since the flow modules are Lipschitz continuous and $\|x\|_2 \leq b$, $\|f_m(\boldsymbol{x})\|_2 \leq L*\|\boldsymbol{x}-\boldsymbol{0}\|_2 + \|f_m(\boldsymbol{0})\|_2=O(b)$. As such, $\|y_t\|_{\infty}\leq \|y_t\|_2\leq O(b)$ and $\|z_t\|_{\infty}\leq \|z_t\|_2\leq O(b)$. Similarly, due to the Lipschitz continuous for $f_{fo_t}$,
$\|\mu_{\phi}(\boldsymbol{x})\|_{\infty}$,$\|\sigma_{\phi}(\boldsymbol{x})^{-1}\|_{\infty} $ can be also bounded by $O(b)$ .\\
For every $ Loss_{f_{foi}}$ outputted by i-th factor-out layer $f_{fo_i} $, according to \textbf{Lemma\ref{lemma1}}, we obtain:\\
 $$
Loss_{f_{foi}}(\boldsymbol{x}_{pert})- Loss_{f_{foi}}(\boldsymbol{x}_{clean})
 = \log(p_{q_{\phi}^{\star}}(\boldsymbol{z}_{i}^{pert}) -\log(p_{q_{\phi}}(\boldsymbol{z}_{i}^{clean}))
 \geq
 -L_2*\delta*\sqrt{dim_i}*O(b^5)
 $$
Further,  assume $f_{mf}: \boldsymbol{y_k} \mapsto \boldsymbol{z_{k+1}} \backsim \mathcal{N}\left(0,I\right)$, similarly,
$z_k$ can be bound by $O(b)$, using \textbf{Lemma\ref{lemma2}}:\\
$$
Loss_{mf}(\boldsymbol{x}_{pert})-Loss_{mf}(\boldsymbol{x}_{clean})
=
\log p(\boldsymbol{z}_{k+1}^{pert})-\log p(\boldsymbol{z}_{k+1}^{clean})
\geq
-L_1*\delta*O(b)-\frac{\delta^2}{2}
$$
Hence:
\begin{eqnarray*}
 L(F(\boldsymbol{x}_{pert}))-L(F(\boldsymbol{x}_{clean}))\nonumber
\geq-\sum_i^kL_2\delta\sqrt{dim_k}*O(b^5)-L_1\delta*O(b)-\frac{\delta^2}{2}
\end{eqnarray*}
\end{proof}

Note that in main-flow of continuous flows we also need to consider the effect of the Jacobian determinant. However, it is necessary to introduce bi-Lipschitz to ensure the local Lipschitz property of the Jacobian determinant.
\begin{mydef}{\rm(Lipschitz Continuity and bi-Lipschitz Continuity )}\\
A fuction $F(\cdot):\mathbb{R}^{n} \rightarrow \mathbb{R}^{n}$ is called $L$-Lipschitz continuous if there exists a constant $L$ such that for any $x_{1}, x_{2} \in \mathbb{R}^{n}$, we have
\begin{equation*}\label{lip}
   \left\|F\left(x_{1}\right)-F\left(x_{2}\right)\right\| \leq L\left\|x_{1}-x_{2}\right\|
\end{equation*}
If an inverse $F^{-1}: \mathbb{R}^{n} \rightarrow \mathbb{R}^{n}$ and a constant $L^{\star}=Lip(F^{-1})$ exists such that for all $y_{1}, y_{2} \in \mathbb{R}^{n}$
$$
\left\|F^{-1}\left(y_{1}\right)-F^{-1}\left(y_{2}\right)\right\| \leq L^{*}\left\|y_{1}-y_{2}\right\|
$$
holds, then $F$ is called bi-Lipschitz continuous.
\end{mydef}

With the bi-Lipshitz Continuity, we obtain:
\begin{lemma}
Assume that $f_m(\cdot): \boldsymbol{x} \mapsto \boldsymbol{z}$ is bi-Lipschitz with bound $L$ and $L^{-1}$, $\boldsymbol{x}\in \mathcal{X}\subset\mathbb{R}^n$. If $\|\boldsymbol{x}^{\star}-\boldsymbol{x}\|_2\leq \delta $, we have:
$$\left|\log \left|\operatorname{det}\left(\frac{\partial\boldsymbol{z}^{\star}}{\partial \boldsymbol{x}^{\star\top}}\right)\right|
-\log \left|\operatorname{det}\left(\frac{\partial\boldsymbol{z}}{\partial \boldsymbol{x}^{\top}}\right)\right| \right|
\leq 2n\log(L)$$
\end{lemma}
\begin{proof}
To prove the Lipschitz bounds, we use the identity:
$$
\operatorname{Lip}(F)=\sup _{x \in \mathbb{R}^{n}}\left\|\operatorname{det}\left(\frac{\partial\boldsymbol{z}}{\partial \boldsymbol{x}^{\top}}\right)\right\|_{2}
$$
where$\|\cdot\|_2$ represents the spectral norm of the Jacobian matrix. Thus, we obtain that:
\begin{equation*}
\begin{aligned}
\left|\log \left|\operatorname{det}\left(\frac{\partial\boldsymbol{z}^{\star}}{\partial \boldsymbol{x}^{\star\top}}\right)\right|
-\log \left|\operatorname{det}\left(\frac{\partial\boldsymbol{z}}{\partial \boldsymbol{x}^{\top}}\right)\right|\right|
&\leq \left|\sum_i^{n}\log |\sigma_{i}(x)|+\sum_i^{n}\log |\sigma_{i}(x^{\star})|\right|\\
&\leq 2n\sup _{x \in \mathbb{R}^{n}}|\log\sigma_{i}(x)|
=2n\log(L)
\end{aligned}
\end{equation*}
where $\sigma_{i}(x)$ denotes the i-th singular value of $\operatorname{det}\left(\frac{\partial\boldsymbol{z}}{\partial \boldsymbol{x}^{\top}}\right)$
\end{proof}

Note that if $F^{-1}$ cannot guarantee Lipschitz property, then the absolute value of the minimal eigenvalue of the determinant can be infinitely close to zero. This makes $|\log\sigma_{i}(x)|$ unbounded. By Theorem 2 in \cite{understanding}, the affine coupling layer does not guarantees the bi-Lipschitz property, but the additive layer does. This indicates that the robustness of affine layer is more fragile in continuous flow, which is consistent with the point we argue in the main text.

\subsection{Section 5: Experimental Visualization}
In this section, we visualize the attack evaluation of AW-PGD on IDF versus Glow. And we visualize our proposed defense methods: R-IDF and R-IDF (Hybrid). 

Due to the limitation of the compression task, no significant gap with PGD is exhibited in Figure\ref{attackdis}, but this situation is mitigated on the Glow model attack, as shown in Figure\ref{attackglow}. Note that AW-PGD does not try to make a small number of samples large to improve the effectiveness of the attack but instead shifts  the data distribution as a whole. Furthermore, from Figure \ref{attackdefence}, we can clearly see that our defense methods have significant robustness improvement. In particular, when the attack strength $\epsilon=1$, R-IDF is not significantly inferior to the adversarial training based method.
\end{spacing}
\begin{figure*}[p]
\begin{center}
\includegraphics[width=1\textwidth]{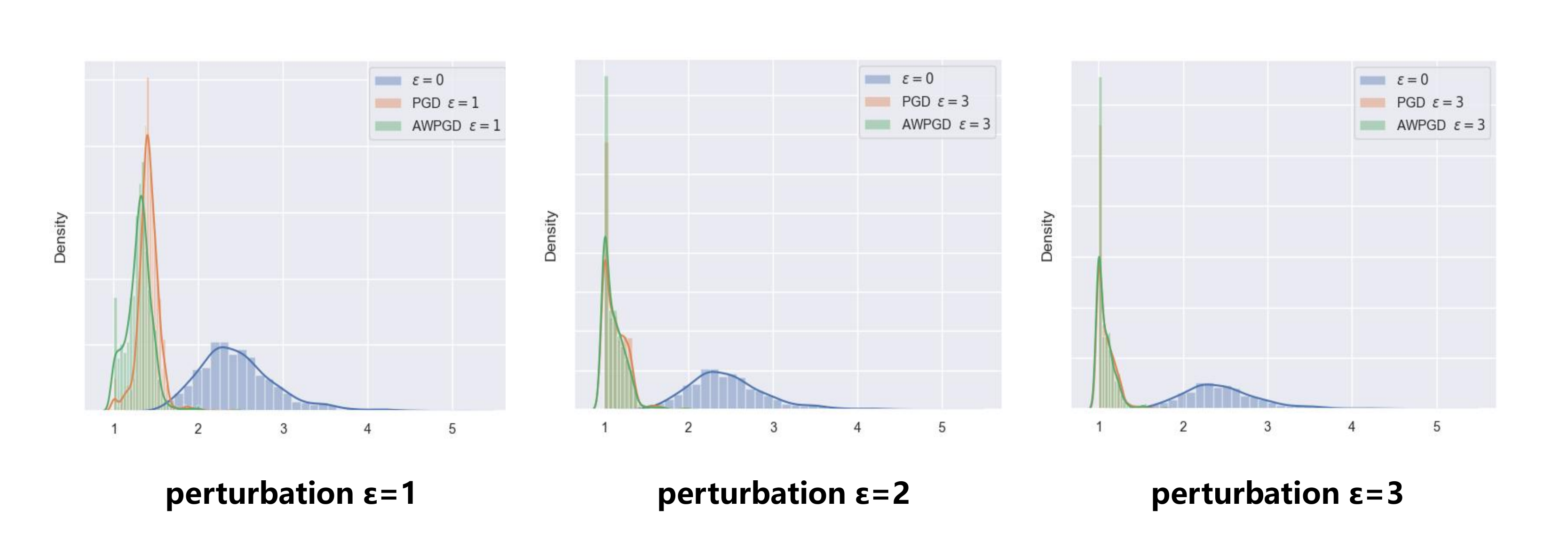}
\end{center}
  \caption{ PGD \& AW-PGD attack distributions for IDF. The x-label is CR, \emph{where lower CR means  worse result.}}
\label{attackdis}
\end{figure*}

\begin{figure*}[p]
\begin{center}
\includegraphics[width=1\textwidth]{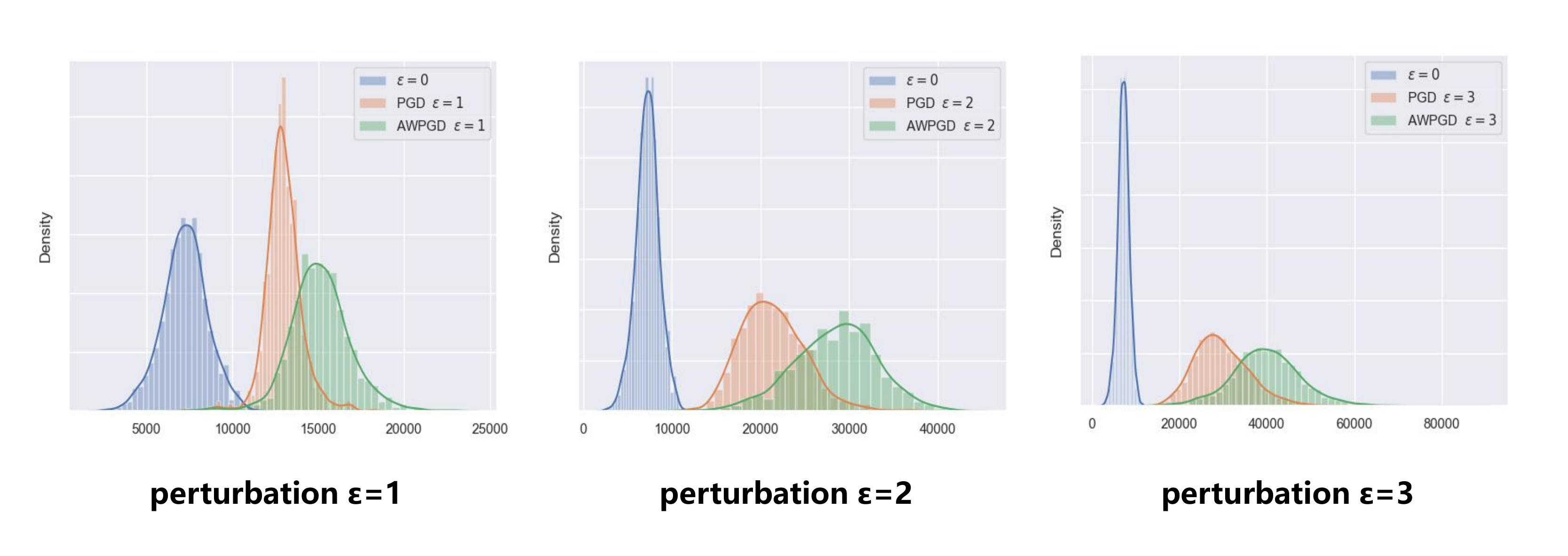}
\end{center}
  \caption{PGD \& AW-PGD attack distributions for Glow. The x-label is NLL, \emph{where higher NLL means  worse result.} }
\label{attackglow}
\end{figure*}

\begin{figure*}[p]
\begin{center}
\includegraphics[width=1\textwidth]{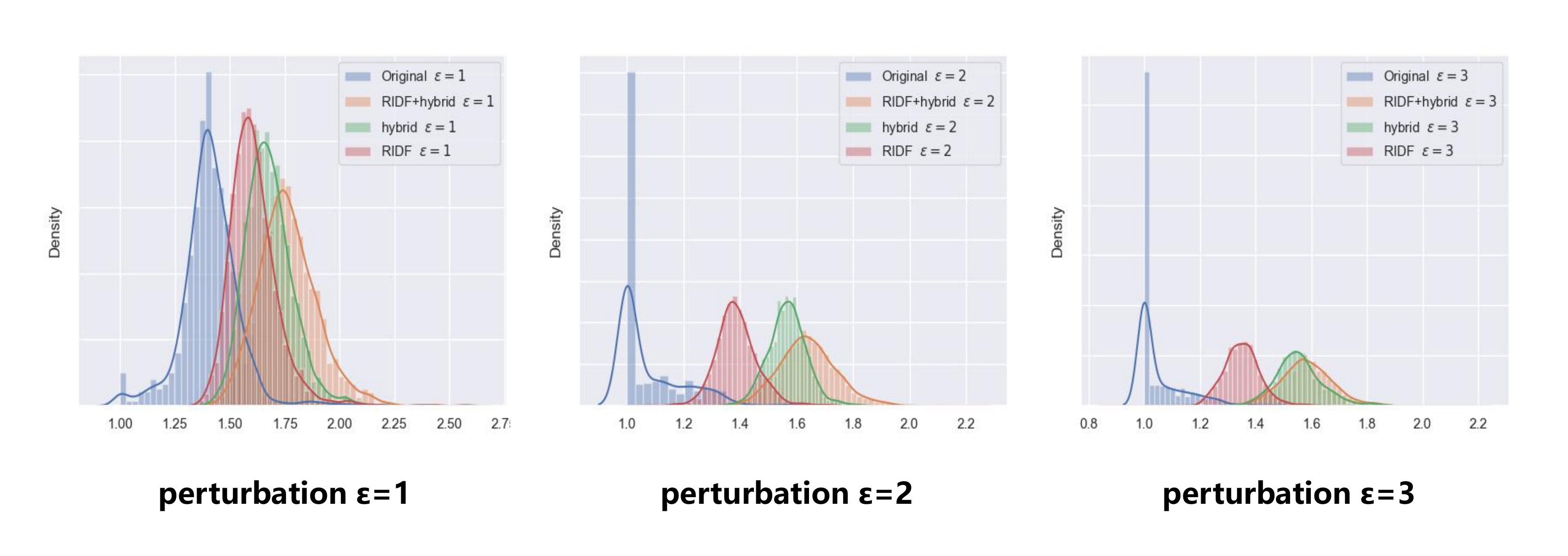}
\end{center}
  \caption{PGD attack distributions for IDF \& R-IDF \& R-IDF(hybrid). The x-label is CR, \emph{where lower CR means  worse result.}}
\label{attackdefence}
\end{figure*}
\end{document}